\documentclass[10pt,a4paper,twocolumn,twoside]{gc/article}
\usepackage[english]{babel}
\usepackage{siunitx}
\usepackage{array}
\usepackage{booktabs}
\lstdefinestyle{python}{
  language=Python,
  basicstyle=\footnotesize\ttfamily,
  keywordstyle=\bfseries,
  commentstyle=\itshape,
  stringstyle=\ttfamily,
  showstringspaces=false,
  breaklines=true,
  frame=none,
  numbers=none,
}
\usepackage{amsthm}

\usepackage{algorithm}
\usepackage{algpseudocode}
\usepackage{amsmath}
\usepackage{amssymb}

\theoremstyle{plain}
\newtheorem{theorem}{Theorem}
\newtheorem{lemma}{Lemma}

\theoremstyle{remark}

\theoremstyle{definition}

\journalname{}
\title{MXNorm: Reusing MXFP block scales for efficient tensor normalisation}

\author{Callum McLean{$^\star$}, Luke Y. Prince{$^\star$}, Alexandre Payot, Paul Balan\c{c}a, and Carlo Luschi}

\institution{Graphcore}

\footinfo{Correspondence: lukep@graphcore.ai}
\theday{\today}
\leadauthor{McLean et al.}

\begin{abstract}    
Matrix multiplication performance has long been the major bottleneck to scaling deep learning workloads, which has stimulated the design of new accelerators that use increasingly low-precision number formats. However, improvements in matrix multiplication performance have far outstripped improvements in performance on reductions and elementwise computations, which are still being performed in higher precision. In this work, we propose MXNorm, a drop-in replacement for RMSNorm that estimates the RMS using only the block scales calculated as part of the MXFP8 cast and enables a 32x decrease in the size of reduction needed for normalization. We validate our approximation method on pre-training of Llama 3 models of 125M, 1B and 8B parameters, finding minimal loss of training accuracy compared to a baseline using RMSNorm with MXFP8 matmuls. We also show practical kernel speedups using only \texttt{torch.compile} of up to 2.4x for MXNorm over RMSNorm, corresponding to a 1.3\% speedup in Llama 3 8B transformer layers in MXFP8 and a 2.6\% speedup in NVFP4.
\end{abstract}


\begin{document}
		
    \maketitle 
    \thispagestyle{firststyle} 
    \gcabstract 
    

\section{Introduction}

The success of deep learning in natural language processing, computer vision, and scientific fields such as molecular biology has been enabled by leaps in AI accelerator capabilities. In particular, an 80x improvement in the GPU acceleration of low-precision matrix multiplications over the past eight years \cite{nvidia_blackwell_2024, scaling-book} (Table \ref{tab:gpu-comparison}) has allowed researchers and practitioners to scale pre-training of transformer-style neural networks and learn from vast pools of unlabelled data. However, as matrix multiplication becomes less of a bottleneck to throughput, other components of model architectures emerge as new bottlenecks. Indeed, other aspects of AI accelerators have not kept pace with improvements in matrix multiplication throughput. For example, elementwise operations and reductions are limited by memory bandwidth and CUDA core throughput in GPUs; however, these have only improved by factors of 8.9x and 5.1x, respectively, over the past 8 years. Furthermore, this gap is set to widen with upcoming GPU architecture releases \cite{nvidia_rubin_2026}. In some cases these operations can be hidden by overlapping them with matrix multiplications \cite{shah2024flashattention3fastaccurateattention}, but others require too much memory to pipeline in practice. As such, we assert that the community needs to consider new building blocks that contribute a smaller overhead.

\newcolumntype{Y}{>{\centering\arraybackslash}m{2.0em}}

\begin{table}[tbh]
\centering
\caption{Comparison of GPU performance and memory bandwidth. Improvement relative to V100 shown in square brackets. dtype: Lowest supported precision of floating point format for matrix multiplication.
*Anticipated specification of Rubin generation GPUs that have yet to be released commercially.
}
\resizebox{\linewidth}{!}{\begin{tabular}{
    ll
    S[table-format=5.0] Y
    S[table-format=3.1] Y
    S[table-format=4.0] Y
}
\toprule
\textbf{GPU} & \textbf{dtype}
& \multicolumn{2}{c}{\textbf{Tensor Core}} 
& \multicolumn{2}{c}{\textbf{CUDA Core}}
& \multicolumn{2}{c}{\textbf{Mem. BW}} \\
& & \multicolumn{2}{c}{(TFLOPS)} 
  & \multicolumn{2}{c}{(TFLOPS)}
  & \multicolumn{2}{c}{(TB/s)} \\
\midrule
V100  & FP16 & 125   &        & 15.7  &        & 0.9  &        \\
A100  & FP16 & 312   & [2.5x] & 19.5  & [1.2x] & 2.0 & [2.2x] \\
H100  & FP8  & 1979  & [15.8x] & 67.0   & [4.3x] & 3.4 & [3.7x] \\
GB200 & FP4  & 10000 & [80.0x] & 80.0   & [5.1x] & 8.0 & [8.9x] \\
Rubin* & FP4 & 35000 & [280x] & 130.0   & [8.3x] & 22.0 & [24.4x] \\ 
\bottomrule
\end{tabular}}
\label{tab:gpu-comparison}
\end{table}

One such building block ripe for consideration is normalization. Normalization layers are essential for ensuring pre-training stability. Various normalization schemes have been favoured over the past several years, such as BatchNorm \cite{ioffe2015batchnormalizationacceleratingdeep} for convolutional neural networks, LayerNorm \cite{ba2016layernormalization} for sequence models, and more recently RMSNorm \cite{zhang2019rootmeansquarelayer} for large language models such as the Llama series \cite{touvron2023llamaopenefficientfoundation, touvron2023llama2openfoundation, grattafiori2024llama3herdmodels}. In the case of RMSNorm, each token's hidden state is normalized using its root mean square. The placement of norms is also key to pre-training performance, with frontier models typically placing norms at the beginning of each residual branch (which we refer to as \emph{Pre-Norm} transformers following \cite{xiong2020layernormalizationtransformerarchitecture}). 

We propose a design for a drop-in replacement for RMSNorm which fuses normalization with conversion to microscaling (MX) formats \cite{rouhani2023microscalingdataformatsdeep}. MX formats are a family of recently proposed block-quantization formats used to accelerate matrix multiplications in the latest generation of AI accelerators by quantizing tensors to fewer than 16 bits per element while preserving the range of higher-precision formats such as BF16 and FP32 (see Appendix \ref{apx:mxfp-primer} and Section \ref{subsec:mxcast} for relevant background). Unlike standard low precision formats (e.g., FP8), this range-preserving property of MX formats can enable practitioners to use direct casting, thereby simplifying training and inference recipes and supporting wider adoption \cite{mishra2025recipespretrainingllmsmxfp8}.

In Pre-Norm transformers (e.g. Llama \cite{touvron2023llamaopenefficientfoundation, touvron2023llama2openfoundation, grattafiori2024llama3herdmodels}, gpt-oss \cite{openai2025gptoss120bgptoss20bmodel}, Deepseek-V3 \cite{deepseekai2025deepseekv3technicalreport}) RMSNorm is executed just before MX quantization. We observe that both MX quantization and RMSNorm gather statistics along the hidden dimension of the tensor to rescale elements.  We also note that when a probability distribution is scaled linearly, the expected absmax of the distribution is scaled accordingly. From these observations we propose approximating the RMS using the block scales calculated during MX quantization, thereby enabling us to fuse RMSNorm with MX quantization for activations and requiring only a single pass of statistics gathering over the whole tensor. We call this new scheme \textbf{MXNorm}, and demonstrate its effectiveness and stability in pre-training of language models of up to 8B parameters and that it provides practical speedups on commercially available hardware at inference time.

\begin{figure*}[ht]
    \centering
    \includegraphics[width=\linewidth]{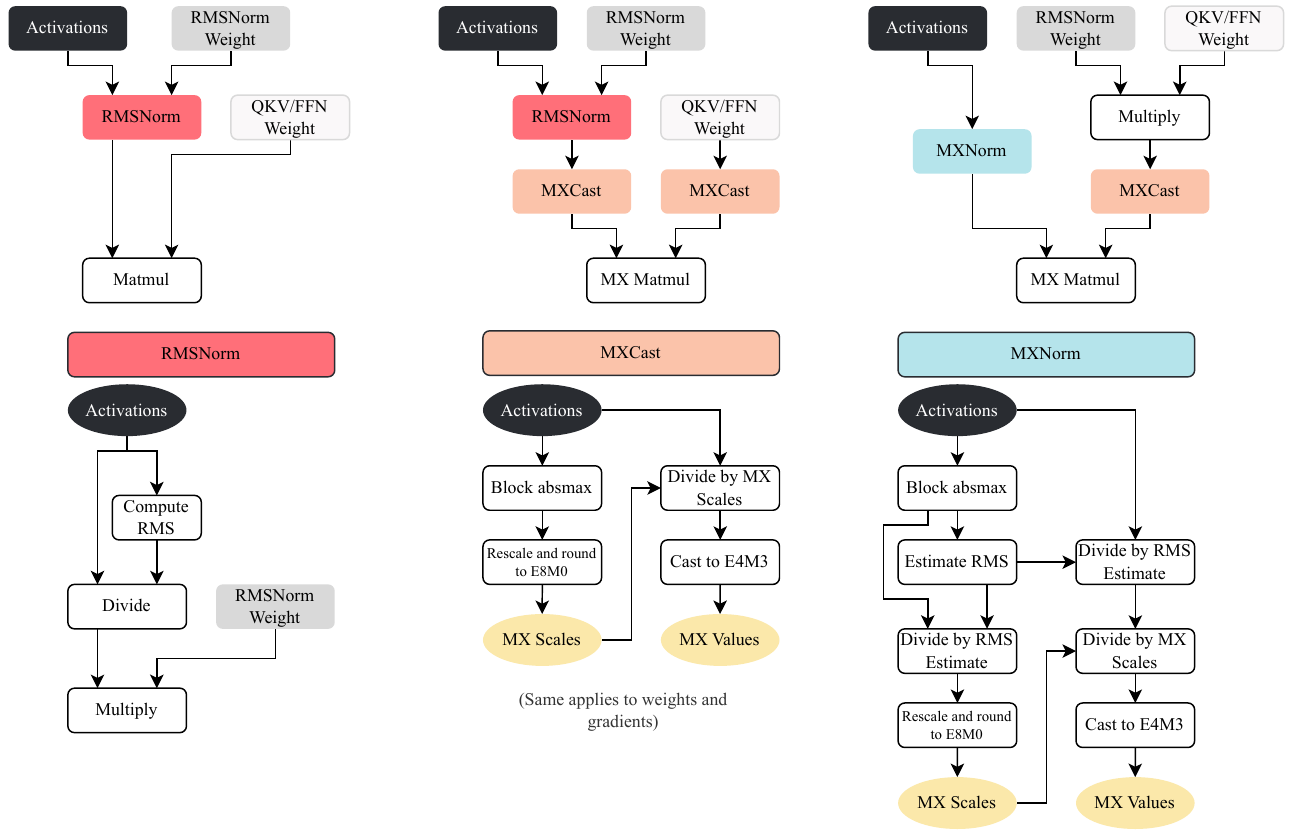}
    \caption{Computational graphs for RMSNorm, MXCast, and MXNorm in the context of Norm + Linear layer pattern. Top left: RMSNorm + Linear graph for high precision training. Top middle: RMSNorm + Linear graph with linear inputs cast to MX (MXLinear). Top right: RMSNorm approximated with MXNorm and RMSNorm weight fused with Linear weight. Bottom left: RMSNorm graph. Bottom middle: MXCast graph for MXFP8 activations. Note that MXCast is applied to weights (E4M3 values) and gradients (E4M3 or E5M2 values) as well. Bottom right: MXNorm graph.}
    \label{fig:comp-graphs}
\end{figure*}

\section{Related Work}

A few efforts have been made to train transformers without normalizing intermediate tensors. One approach has been to instead constrain weights to lie on a hypersphere, which incurs costs during optimizer steps \cite{loshchilov2025ngptnormalizedtransformerrepresentation, pethick2025trainingdeeplearningmodels}. More recently, \cite{zhu2025transformersnormalization} demonstrated the possibility of replacing normalization layers with a saturating non-linearity, at the cost of using slow elementwise special operations.

At time of writing, we are aware of two attempts to accelerate normalization, (1) by computing the RMS asynchronously and normalizing after multiplying the raw input by the weight \cite{graef2025flashnormfastnormalizationllms}, which runs the risk of swamping accumulators during matrix multiplication and harming output accuracy, and (2) by computing the RMS with the first $k$ vector elements \cite{zhang2019rootmeansquarelayer}, which has a high probability of missing outlier values that dominate the RMS.

\section{Methods} \label{sec:methods}

Here we will briefly describe RMSNorm, conversion to MXFP, and our main contribution: MXNorm. The differences between each of these schemes and their use in a Norm + Linear layer is summarized in Figure \ref{fig:comp-graphs}. We also provide a diff-like pseudocode comparison between RMSNorm followed by an MX cast in Algorithm \ref{algo:rmstomx} and MXNorm in Algorithm \ref{algo:mxnorm}.

\begin{figure*}[h]
\centering
{\fontsize{10}{12}\selectfont  
\begin{minipage}[t]{0.48\textwidth}
\begin{algorithm}[H]
\caption{RMSNorm + MXCast}
\begin{algorithmic}[1]
\Statex \hspace{-1.5em} \textbf{Input:} $X \in \mathbb{F}_{\star}[T \times K \times B]$
\Statex \hspace{-1.5em} \textbf{Output:} $S \in \mathbb{F}_{E8M0}[T \times K], \: V \in \mathbb{F}_{E4M3}[T \times K \times B]$ 
\Statex \hspace{-1.5em} 
\Statex \hspace{-1.5em} $\rho_t \gets \left( \frac{1}{KB} \sum_{kb} X_{tkb}^2 \right)^{-1/2}$ \hfill \textcolor{purple}{\# $\rho_t  = 1/RMS$}
\Statex \hspace{-1.5em} $\bar X_t \gets X_t \cdot \rho_t$
\Statex \hspace{-1.5em} $m_{tk} \gets \max_b|\bar X_{tkb}|$
\Statex \hspace{-1.5em} \textcolor{white}{$\tilde \rho_t \gets \left( \sum_{kb} X_{tkb}^p \right)^{-1/p}\cdot c$} 
\Statex \hspace{-1.5em} $S_{tk} \gets \text{cast}(m_{tk} / 256; E8M0)$
\Statex \hspace{-1.5em} $V_{tkb} \gets \text{cast}(\bar X_{tkb} / S_{tk}; E4M3)$
\end{algorithmic}
\label{algo:rmstomx}
\end{algorithm}
\end{minipage}%
\hfill
\begin{minipage}[t]{0.48\textwidth}
\begin{algorithm}[H]
\caption{MXNorm}
\begin{algorithmic}[1]
\Statex \hspace{-1.5em} \textbf{Input:} $X \in \mathbb{F}_{\star}[T \times K \times B]; c, p$
\Statex \hspace{-1.5em} \textbf{Output:} $\tilde{S} \in \mathbb{F}_{E8M0}[T \times K], \: \tilde{V} \in \mathbb{F}_{E4M3}[T \times K \times B]$
\Statex \hspace{-1.5em}
\Statex \hspace{-1.5em} \textcolor{white}{$\rho_t \gets \left( \sum_{kb} X_{tkb}^2 \right)^{-1/2}$} 
\Statex \hspace{-1.5em} \textcolor{white}{$\bar X_t \gets X_t \cdot \rho_t$} 
\Statex \hspace{-1.5em} $\tilde m_{tk} \gets \max_b|X_{tkb}|$
\Statex \hspace{-1.5em} $\tilde \rho_t \gets \left(\frac{1}{K} \sum_{k} \tilde m_{tk}^p \right)^{-1/p}\cdot c$ \hfill $\textcolor{purple}{\text{\# choose } c, p \text{ so } \tilde \rho_t \approx \rho_t}$
\Statex \hspace{-1.5em} $\tilde{S}_{tk} \gets \text{cast}(\tilde m_{tk} \cdot \tilde \rho_t/ 256; E8M0)$
\Statex \hspace{-1.5em} $\tilde{V}_{tkb} \gets \text{cast}(X_{tkb}\cdot \tilde \rho_t / S^{\tilde{X}}_{tk}; E4M3)$
\end{algorithmic}
\label{algo:mxnorm}
\end{algorithm}
\end{minipage}
}  
\caption*{Side by side diff-like comparison of RMSNorm + MXCast against MXNorm without norm gain parameters. Input $X$ with shape $[T \times K \times B]$ is received in blocked format where it originally had $D$ columns that have been partitioned in $K$ blocks of size $B$. We omit the use of a normalization stability constant $\epsilon$ for brevity. Block absmaxes are divided by 256 (the largest representable power of 2 in E4M3) to ensure maximal use of range. $p$ is a scalar norm exponent, used with p=1 for the arithmetic mean and p=2 for the RMS (quadratic mean) of absmaxes. $c$ is a scalar correction factor dependent on block size $B$ and $p$ to ensure $\tilde \rho \approx \rho$ when $X_t$ are i.i.d. samples from a Gaussian distribution. $\text{cast}(\cdot; EaMb)$ represents a cast to floating point format $\mathbb{F}_{EaMb}$ with $a$ exponent and $b$ mantissa bits. E8M0 casting uses an implementation defined power-of-2 rounding function $2^{\lfloor \log_2\cdot \rceil}$ \cite{or2025torchaopytorchnativetrainingtoservingmodel}. $\mathbb{F}_{\star}$ corresponds to a higher-precision floating point format.}
\end{figure*}

\subsection{RMSNorm} \label{sec:rmsnorm}

Given a tensor $X \in \mathbb{R}^{T \times D}$, RMSNorm normalizes each row of the tensor $X$ using the inverse RMS $\rho$, given by:
\begin{align}
\rho_t &= \left(\frac{1}{D}\sum_{d=1}^D X_{td}^2\right)^{-1/2} \label{eqn:rms}
\end{align}
The normalized input $\bar X_{td} = \rho_t X_{td}$ is rescaled by the learnable gain parameter $\gamma \in \mathbb{R}^D$ along the hidden dimension to give output $Y$ as follows:
\begin{align}
    Y_{td} &= \rho_t \cdot X_{td} \cdot \gamma_d \label{eqn:gain}
\end{align}
In Pre-Norm architectures \cite{xiong2020layernormalizationtransformerarchitecture}, there is an RMSNorm layer immediately prior to the QKV projection in each attention layer and immediately prior to the input and gate projections of the FFN (which can be fused into a single matrix multiplication). This leads to the ``RMSNorm $\rightarrow$ Linear" pattern shown in the top-left of Figure \ref{fig:comp-graphs}.

\subsection{Conversion to MXFP8 (MXCast)} \label{subsec:mxcast}

For conversion to MX we first partition the $D$ elements of each row of tensor $X$ into $K$ blocks of size $B$. We need to compute a per block 8-bit quantized scale $S \in \mathbb{F}^{T \times K}_{E8M0}$ that is used in turn to compute 8-bit quantized values $V \in \mathbb{F}^{T \times K \times B}_{E4M3}$, where $\mathbb{F}_{EaMb}$ denotes the set of values that can be represented in the floating point format $EaMb$.

To compute block scales $S_{tk}$ we first compute block absmaxes $m_{tk} = \max_b|Y_{tkb}|$, rescale by the largest power of 2 representable by $V$ to maximize usage of the range (e.g., 256 for E4M3), then round to a power of 2.

\begin{equation} \label{eqn:blockscale}
    S_{tk} = \text{cast}(m_{tk} / 256; E8M0)
\end{equation}

Power-of-two rounding is implementation-defined, with possible options given by (for example) \cite{rouhani2023microscalingdataformatsdeep, pytorch_ao_mx_config_2024, mishra2025recipespretrainingllmsmxfp8}. For our experiments we use the method defined in \cite{mishra2025recipespretrainingllmsmxfp8}, enumerated as \texttt{ScaleCalculationMode.RCEIL} in TorchAO \cite{or2025torchaopytorchnativetrainingtoservingmodel}.

We can then quantize block values by rescaling elements using block scales and casting to the target 8-bit floating point format.

\begin{equation} \label{eqn:blockvals}
    V_{tkb} = \text{cast}(Y_{tkb} / S_{tk}; E4M3).
\end{equation}

Where $\text{cast}(\cdot; EaMb)$ represents a cast to a (signed) floating-point format with $a$ exponent bits and $b$ mantissa bits. 

\subsection{Approximation of RMS using block absmaxes} \label{sec:rms_approx}

To motivate the design of MXNorm, we consider how we can approximate the inverse RMS as a function of the block absmaxes $\tilde \rho = f(m)$.

We first note that a generalized power mean of block absmaxes can be rescaled linearly to estimate the RMS. Intuitively, if a vector is scaled by a scalar factor $\sigma$ then both the RMS and the generalized power mean of the elements are scaled by $\sigma$, and so the ratio between them should be constant. Indeed, we can prove the following:

\begin{theorem} \label{thm:mean-absmax-almost-sure-rms-ratio}
Fix a block size $B$. Let $(X_i)_{i=1}^D$ be $D = KB$ i.i.d.\ samples from a scale family distribution such that $X = \sigma Z$, where $\sigma > 0$ and $Z$ satisfies $\mathbb{E}[Z^2] = 1$ and $\mathbb{P}(Z=0) = 0$. Partition the indices $\{1,\dots,D\}$ into $K$ disjoint blocks of size $B$, and define the block absolute maxima
\[
m_{k} \coloneqq \max_{1 \le b \le B} |X_{kb}|, \qquad k = 1,\dots,K.
\]
For $p > 0$, define the generalized $p$-mean of these block maxima by
\begin{equation}
G^{(p)}_K \coloneqq 
\left( \frac{1}{K} \sum_{k=1}^K m_{k}^p \right)^{1/p}.
\label{eq:GpK}
\end{equation}
Then, as $K \to \infty$,
\[
\frac{G^{(p)}_K}{\operatorname{RMS}(X)} \;\xrightarrow{\mathrm{a.s.}}\;
c^{(p,B)}
\]
where
\begin{equation}
c^{(p,B)} \coloneqq 
\left(
\mathbb{E}_k\!\left[
\max_{b} |Z_{kb}|^p
\right]
\right)^{1/p}
\label{eq:cpK}
\end{equation}
and $(Z_i)$ are i.i.d.\ copies of $Z$, provided that $\mathbb{E}_k\!\left[\max_{b} |Z_{kb}|^p \right] < \infty$.  
Thus the generalized $p$-mean of block absolute maxima converges to the RMS up to a multiplicative constant depending only on $p$, $B$, and the standardized distribution of $Z$.
\end{theorem}

\begin{proof}

Observe that
\[
\operatorname{RMS}(X)^2
= \frac{1}{D}\sum_{d=1}^D X_i^2
\xrightarrow{\mathrm{a.s.}} \mathbb{E}[X^2]
\]
by the strong law of large numbers, and that
\[
\mathbb{E}[X^2]
= \sigma^2 \mathbb{E}[Z^2] = \sigma^2
\]

Hence $\operatorname{RMS}(X) \xrightarrow{\mathrm{a.s.}} \sigma $ by the continuous-mapping theorem.

For each block, homogeneity (meaning $f(a\cdot U) = a\cdot f(U)$) of the absmax yields
\[
m_{k} = \max_{b} |\sigma Z_{kb}|
= \sigma \max_{b} |Z_{kb}|.
\]
Define
\[
\mu_{k} \coloneqq \max_{b} |Z_{kb}|.
\]
The variables $\mu_{k}$ are i.i.d.\ and depend only on $B$ and the distribution of $Z$.

Substituting $m_{k} = \sigma \mu_k$ into the definition of $G^{(p)}_{k}$ in \eqref{eq:GpK},
\[
G^{(p)}_{K}
= \sigma \left( \frac{1}{K} \sum_{k=1}^K \mu_{k}^p \right)^{1/p}.
\]

$\mathbb{E}_k[\mu_{k}^p] < \infty$, so by the strong law of large numbers we have
\[
\frac{1}{K} \sum_{k=1}^K \mu_{k}^p
\;\xrightarrow{\mathrm{a.s.}}\;
\mathbb{E}_k[\mu_{k}^p]
=
\mathbb{E}_k\!\left[
\max_{b} |Z_{kb}|^p
\right].
\]
Applying the continuous-mapping theorem with $f(x) = x^{1/p}$ and using \eqref{eq:cpK} gives
\[
G^{(p)}_{K}
\;\xrightarrow{\mathrm{a.s.}}\;
\sigma\, c^{(p,B)}.
\]

$\mathbb{P}(Z = 0) = 0$, so $\mathbb{P}(\operatorname{RMS}(X) > 0) = 1$ and so applying the continuous-mapping theorem with $f(a, b) = a/b$ (which is continuous on $C =\mathbb{R} \times \mathbb{R}^+$) gives the desired result.

\end{proof}

We state the continuous-mapping theorem and discuss the premise that $\mathbb{P}(Z=0) = 0$ in Appendix \ref{apx:proof-mean-absmax-almost-sure-rms-ratio}.

\subsection{MXNorm}

From the insight in Section \ref{sec:rms_approx}, we can build a normalization function using the generalized $p$-mean over block absmaxes to replace $\rho$ with $\tilde \rho$ in Equation \ref{eqn:rms} for an activation tensor  $X \in \mathbb{R}^{T\times K\times B}$ reshaped into blocks,

\begin{align}
    \tilde m_{tk} &= \max_b X_{tkb} \\
    \tilde \rho_t &= \tilde c^{(p, B)} \cdot \left(\frac{1}{K}\sum_{k=1}^K \tilde m_{tk}^p\right)^{-1/p} + \epsilon
\end{align}

where $\tilde c^{(p, B)}$ is an estimate of $c^{(p, B)}$ calculated using Monte Carlo samples from a Gaussian distribution and the addition of $\epsilon << 1$ prevents division by zero.

We can then modify Equations \ref{eqn:blockscale} and \ref{eqn:blockvals} to incorporate our approximated inverse RMS $\tilde \rho$ for fusing quantization and normalization of $X$ to $(\tilde{S}, \tilde{V})$,

\begin{align}
    \tilde S_{tk} &= \text{cast}(\tilde m_{tk} \cdot \tilde \rho_t / 256; E8M0) \\
    \tilde V_{tkb} &= \text{cast}(X_{tkb} \cdot \tilde \rho_t/ \tilde S_{tk}; E4M3).
\end{align}

In this paper we experiment with MXNorm variants using generalized $p$-mean for p=1 and p=2. As an ablation of our design, we also consider a scheme in which we estimate the RMS using power-of-2 rounded block absmaxes (i.e. the E8M0 scale factors). The details and results of this ablation can be found in Appendix \ref{apx:post-round}.

\subsection{Bounds on output of RMSNorm and MXNorm}

The most well-known property of RMSNorm is that it normalizes a tensor to have an RMS of 1. An often-overlooked property is that the maximum value of the output of RMSNorm is bounded. This was also observed by \cite{zhu2025transformersnormalization} where they demonstrate that RMSNorm can be replaced by a saturating non-linearity like $\tanh$.

Here we compare the bounds of RMSNorm and MXNorm and show how they are related.

\begin{lemma}[Upper bound for RMS normalization]
For any nonzero $x \in \mathbb{R}^N$, the RMS-normalized vector
\[
\hat{x} = \frac{x}{\operatorname{RMS}(x)}
\]
satisfies
\[
\|\hat{x}\|_\infty \le \sqrt{N}.
\]
\end{lemma}

\begin{proof}
Using $\sum_{i=1}^N x_i^2 \geq \max_i x_i^2$ we can obtain
\[
\operatorname{RMS}(x)
= \sqrt{\frac{1}{N} \sum_{i=1}^N x_i^2 }
\ge \frac{1}{\sqrt{N}} \max_i |x_i|,
\]
we have
$\max_i |x_i| \le \sqrt{N}\,\operatorname{RMS}(x)$.
Dividing by $\operatorname{RMS}(x)$ yields
$\|\hat{x}\|_\infty \le \sqrt{N}$.
\end{proof}




A similar argument can be made with MXNorm using a generalized $p$-mean over the block absmaxes. 

\begin{lemma}[Upper bound for MXNorm normalization]
Let $x \in \mathbb{R}^N$ be partitioned into $K$ blocks of size $B$ and let $G^{(p)}(x)$ be defined as in~\eqref{eq:GpK}.
Then the MXNorm-normalized vector
\[
\tilde{x} = \frac{x}{G^{(p)}(x)}
\]
satisfies
\[
\|\tilde{x}\|_\infty \le \frac{K^{1/p}}{c^{(p,B)}}.
\]
\end{lemma}

\begin{proof}
Let $A = (A_1,\dots,A_K)$ be the vector of block absmaxes.
Then

\[
G^{(p)}(x) 
= c^{(p,B)} \cdot \left(\frac{1}{K} \sum_{k=1}^K A_k^p \right)^{1/p} 
\geq \frac{c^{(p,B)}}{K^{1/p}} \max_i |x_i|
\]

since $\max_b A_b = \max_i |x_i|$. As such we have $\max_i |x_i| \leq \frac{K^{1/p}G^{(p)}(x) }{c^{(p,B)}}$. Dividing by $G^{(p)}(x)$ yields $\|\tilde{x}\|_\infty \leq \frac{K^{1/p}}{c^{(p,B)}}$
\end{proof}

\subsection{Approximation quality of MXNorm}

We demonstrate the fidelity of our approximation by comparing the distributions of scales and value tensors produced by MXNorm against those produced by RMSNorm followed by MXCast. Inputs are generated from centred Gaussian distributions with different standard deviations. We provide the scale and value tensor distributions of an unnormalized input as a reference. The left two panels of Figure \ref{fig:rms-approx} shows that the distributions of scales and values are almost identical. We also demonstrate that the approximation quality of dequantized MX tensors as measured by $r^2$ improves asymptotically towards 1 as the number of blocks increases (see Figure \ref{fig:rms-approx}, right), with excellent approximation quality for as few as 1024 elements (32 blocks of 32).

\begin{figure}[htbp]
    \centering
    \includegraphics[width=\linewidth]{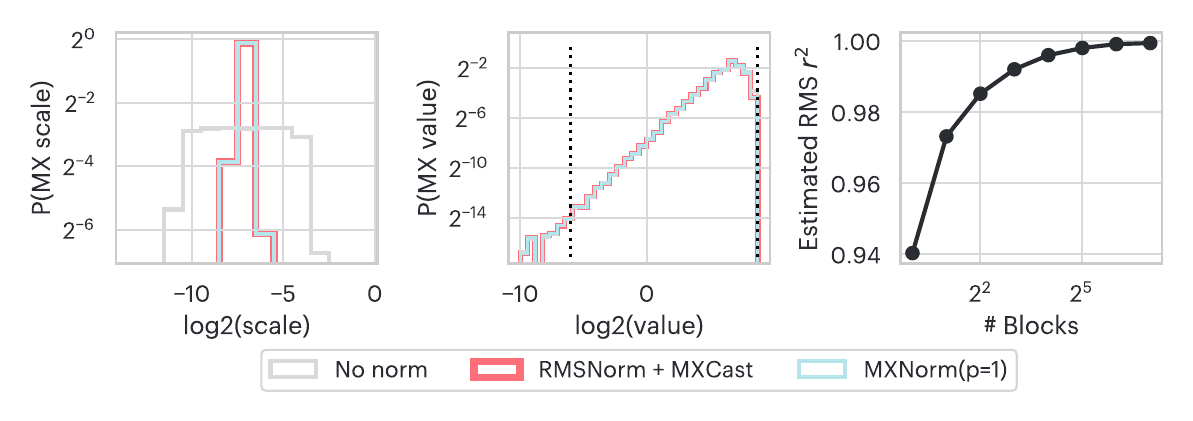}
    \caption{MXNorm as an approximation of RMSNorm. Left: MX scale distribution of normalized tensors. Middle: MX value distribution of normalized tensors. Right: MXNorm $r^2$ goodness-of-fit approaches 1 with more blocks.}
    \label{fig:rms-approx}
\end{figure}

\subsection{MXNormLinear} \label{sec:mxnormlinear}

Up until this point, we have yet to define how to parameterize MXNorm with an affine gain parameter $\gamma$ as is commonly done with RMSNorm and other normalization layers. This poses a difficulty for MXNorm in that it is not trivial to define a broadcasted elementwise multiplication on a tensor represented in an MX format (that is, as a tuple of block scales and values).

However, we can exploit the associativity of linear operations and the fact that normalization layers are typically followed by linear layers, and therefore apply the norm gains to the following weight matrices instead. As such we can define the quantized matrix multiplication to compute output $H$ at the core of a MXNormLinear layer as

\begin{align} \label{mx-matmul}
    H &= \texttt{MXNorm}(X) \: \texttt{MXCast}(W \gamma)^{\top}
\end{align}

where $H$ is accumulated in higher precision.

To train models using MXNorm, we must also be able to compute its gradients. To ensure smoothness of gradients we re-use the gradient calculation of RMSNorm as a straight-through estimator. We provide a full definition of our straight-through estimator in Appendix \ref{apx:rmsnorm-bwd}. 

A standard autograd package of RMSNorm used in conjunction with MX quantization would calculate weight gradients as

\begin{equation}
    \nabla W = \texttt{MXCast}(\nabla H^ \top)\:\texttt{MXCast}(\bar X \gamma)
\end{equation}

In this case we must requantize $\bar X \gamma$ as it will be block-quantized along rows in the forward pass, but must be block-quantized along columns for the backward pass. Similarly, we cannot use the output of MXNorm as it will be block-quantized along rows.

Instead, we cache the input $X$ and the inverse RMS estimate $\tilde \rho$ to compute a high precision $\tilde X = \tilde\rho X$ when needed. As such the new gradient calculation for a linear layer using MXNorm will be

\begin{equation}
    \nabla W = \texttt{MXCast}(\nabla H^ \top)\: \texttt{MXCast}(\tilde \rho \cdot X)\cdot\gamma
\end{equation}

Note that we move the application of the norm gain parameter $\gamma$ outside quantization, as would be performed by standard autograd packages for Equation \ref{mx-matmul}.

Since RMSNorm used in conjunction with MXCast would require caching a high-precision input to a matmul, our approach incurs no further memory overheads. Furthermore, in cases where weight matrices are larger than activation tensors, we can amortize any overhead caused by moving the multiplication of norm gains by (i) performing multiplication on sharded weights during training, (ii) performing the multiplication before gradient accumulation, and/or (iii) combining the trained norm gain and weight matrix parameters for inference.

 We provide a PyTorch implementation of the MXNormLinear forward and backward pass in Appendix \ref{apx:impl}.

\section{Experiments and Results} \label{sec:experiments-and-results}

\subsection{Pre-training} \label{subsec:pretraining}
\begin{figure}[htbp]
    \centering
    \includegraphics[width=\linewidth]{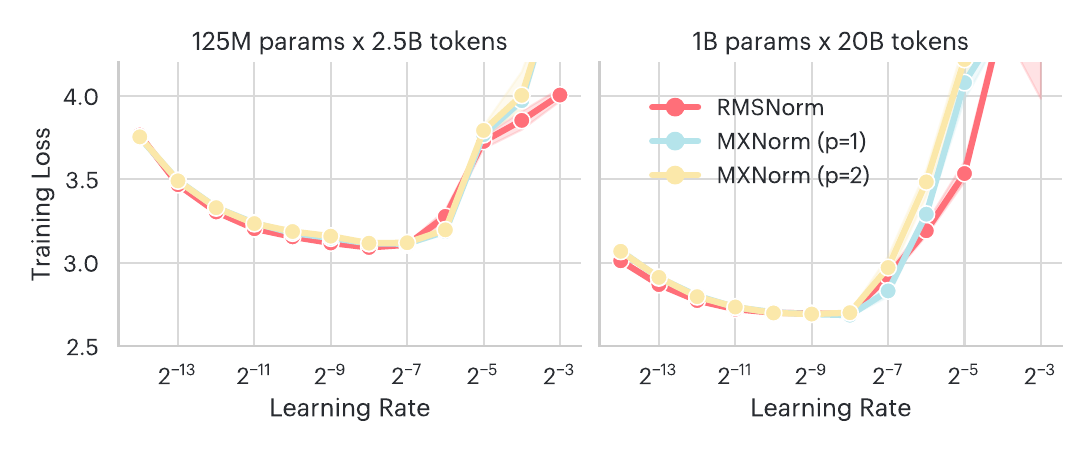}
    \caption{Learning rate sensitivity of MXNorm compared to RMSNorm. Left: 125M parameter model (depth=8, width=1024). Right: 1B parameter model (depth=16, width=2048).}
    \label{fig:lr-sweep}
\end{figure}

We validate MXNorm on pre-training of Llama 3 models \cite{grattafiori2024llama3herdmodels} of different sizes on the SlimPajama dataset \cite{cerebras2023slimpajama}, comparing against a baseline of RMSNorm followed by MXLinear layers. For all experiments we use the TorchTitan distributed pre-training library \cite{liang2025torchtitanonestoppytorchnative} with FSDP in conjunction with TorchAO \cite{or2025torchaopytorchnativetrainingtoservingmodel} for MX quantization. For learning rate sweeps we increment in powers of 2 from $2^{-14}$ to $2^{-3}$. In all other cases we reuse the default hyperparameters used for Llama 3 pre-training in TorchTitan using FSDP and BFloat16. See Appendix \ref{apx:experiment_details} for more details of the hyperparameters used for our experiments.

We examine pre-training stability at small scale by running a learning rate sweep on 125M parameter and 1B parameter models. The effect of quantization is often felt on training stability and can be seen at smaller scales by examining learning rate sensitivity \cite{wortsman2023smallscaleproxieslargescaletransformer}.

In Figure \ref{fig:lr-sweep}, we demonstrate that around the optimal learning rate there is little difference in the training loss between RMSNorm (125M: 3.090±0.004, 1B: 2.692±0.011) for MXNorm using either p=1 (125M: 3.113±0.012, 1B: 2.684±0.009) or the p=2 scheme (125M: 3.116±0.010, 1B: 2.691±0.007). For larger learning rates, MXNorm degrades more quickly than RMSNorm but not catastrophically.  This indicates that MXNorm is broadly stable regardless of power mean at small scale. We show loss curves for the optimum learning rate of 1B models in Appendix \ref{apx:convergence}.

\begin{figure}[htbp]
    \centering
    \includegraphics[width=\linewidth]{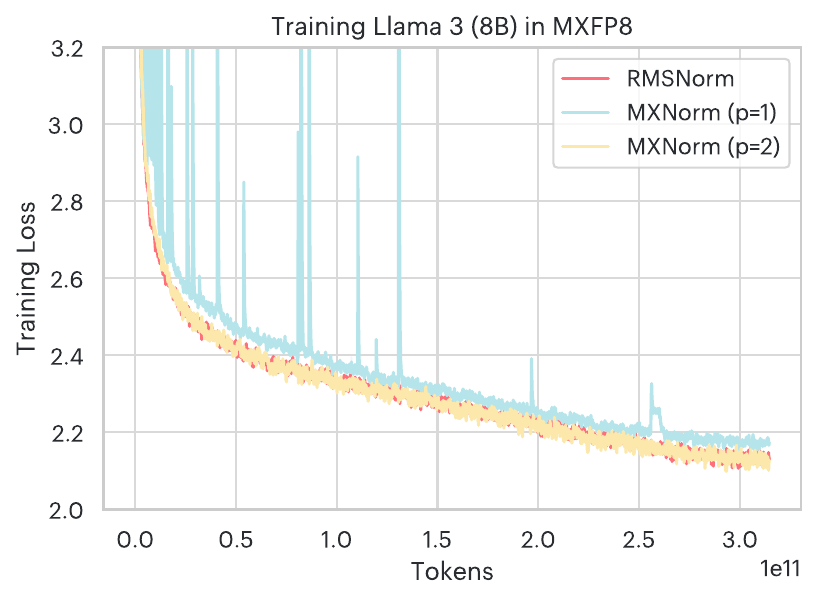}
        \caption{Training loss convergence of 8B parameter models trained on 300B tokens with MXNorm and RMSNorm}
    \label{fig:training-loss-8b}
\end{figure}

We test each scheme on 8B parameter models trained on 300B tokens and show loss curves in Figure \ref{fig:training-loss-8b}. Despite little difference at smaller scale, here we demonstrate that MXNorm(p=1) attains a significantly worse final loss (2.175) compared to RMSNorm (2.132). In contrast, MXNorm(p=2) matches the training loss of RMSNorm  (2.126). Comparable training loss also translates to comparable zero-shot performance on OLMES natural language processing tasks \cite{gu2025olmesstandardlanguagemodel}, with MXNorm(p=2) and RMSNorm both leading in 5/10 benchmarks (Table \ref{tab:olmes-eval}).

\begin{table*}[htbp]
\centering
\caption{Zero-shot performance on OLMES natural language tasks for 8B parameter models}
\begin{tabular}{lrrrrrrrrrr}
\toprule
 Norm Type     &   ARC-C        &   ARC-E       &   BoolQ       &   CSQA        &   HellaSwag &   MMLU &   OBQA &   PIQA &   SIQA &   WinoGr \\
\midrule
 RMSNorm       &  45.3          & \textbf{75.9} & \textbf{73.2} & \textbf{68.5} &        70.4 &   37.9 &          50.4 &   74.8 &        \textbf{56.5} &     \textbf{66.8} \\
 MXNorm (p=1) &  39.1          &       73.7 &    56.6 &            62.4 &        68.1 &   37.2 &          49.8 &   74.5 &        52.5 &     64.4 \\
 MXNorm (p=2)  &  \textbf{45.8} &       74.1 &    72.6 &            67.1 &        \textbf{71.5} &   \textbf{38.0} &          \textbf{54.2} &   \textbf{76.2} &        55.5 &     65.2 \\
\bottomrule
\end{tabular}
\label{tab:olmes-eval}
\end{table*}

\subsection{Activation statistics and loss spike analysis}

To understand why MXNorm(p=1) failed and MXNorm(p=2) succeeded in scaling to 8B parameters we first examine the events leading up to the first spike in MXNorm(p=1). Before the spike, the loss curves are largely the same. Afterwards, the loss for MXNorm(p=1) lags behind that of RMSNorm, suggesting that optimization during a critical early phase of training was missed during the events of the spike. We point to the sudden appearance of outlier features as a proximal cause of the loss spike, as demonstrated in Figure \ref{fig:outlier-features}. In this example, the value across all tokens of a single hidden feature explodes in the step where a loss spike starts. Outlier features are a well-known phenomenon in LLM pre-training \cite{dettmers2022llmint88bitmatrixmultiplication, sun2024massiveactivationslargelanguage} and are typically mitigated by stabilization mechanisms such as tensor normalization \cite{henry2020querykeynormalizationtransformers}, score capping \cite{gemmateam2024gemma2improvingopen}, regularization \cite{chowdhery2022palmscalinglanguagemodeling}, or attention sinks \cite{sun2024massiveactivationslargelanguage, xiao2024efficientstreaminglanguagemodels}.

\begin{figure}[htbp]
    \centering
    \includegraphics[width=\linewidth]{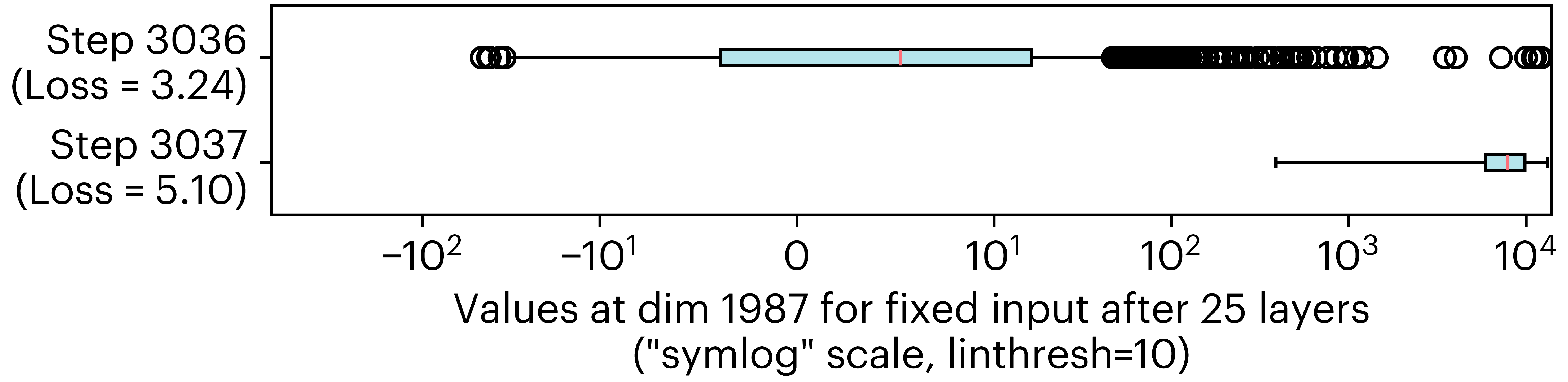}
    \caption{Example outlier feature that appears at the same step as the first loss spike in a 8B parameter training run with MXNorm using mean over block absmax to estimate RMS.}
    \label{fig:outlier-features}
\end{figure}

To understand why MXNorm(p=1) fails to stabilize pre-training, we investigate activation statistics of a 250M parameter proxy model with 16 layers and a hidden dimension of 1024 trained on 5B tokens. We also reduce MX block size to 16 in order to have the same block count as a 1B parameter model. In this setting we set the learning rate to $2^{-9}$, AdamW $\beta_2$ to 0.98, and global batch size to 128 to push training dynamics into a more unstable regime. In this setting, MXNorm(p=1) reliably has more loss spikes than MXNorm(p=2) and RMSNorm (Figure \ref{fig:spike-scores}).

Two possible causes are that (1) MXNorm(p=1) provides a bad estimate of the RMS close to loss spike causing values to explode, or (2) the larger bounds on the output of MXNorm(p=1) lead to more unstable training dynamics.

To investigate the possibility of (1) we track the RMS of the output across all layers (see Figure \ref{fig:norm-output-rms}). The envelope of the loss spike can be seen as a downward deflection of the RMS below 1.0 just before $10^9$ tokens. Before the loss spike, MXNorm(p=1) provides a highly accurate estimate of the RMS since the output RMS is close to 1.0. This suggests that providing an accurate estimate of the RMS is insufficient for stabilizing pre-training.

\begin{figure}
    \centering
    \includegraphics[width=\linewidth]{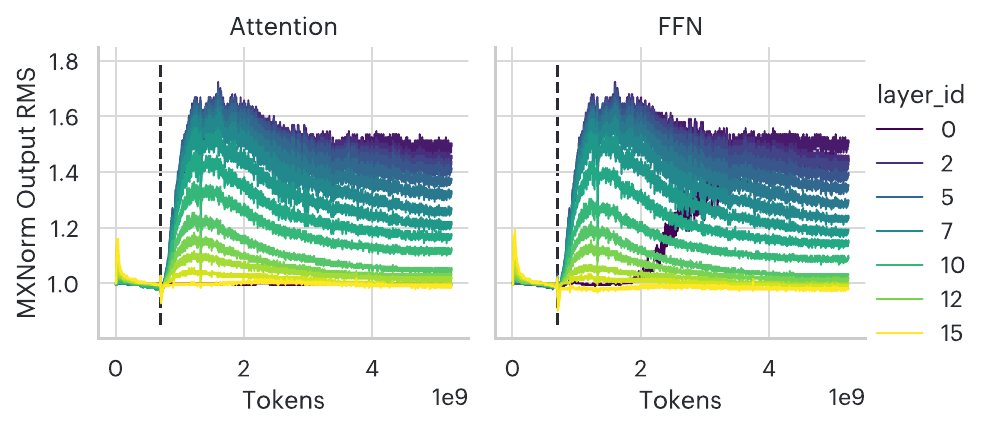}
    \caption{RMS of MXNorm outputs tracked throughout training across all 16 layers of small-scale proxy used to reproduce loss spikes. Envelope of loss spike is visible as a deflection below 1.0 just before $10^9$ tokens. Prior to the loss spike, normalization by MXNorm closely matches behaviour of RMSNorm.}
    \label{fig:norm-output-rms}
\end{figure}

To explore the possibility of (2) we examine histograms of input and output tensor elements dequantized to BF16 for each norm type at the end of training (Figure \ref{fig:norm-io-hists}). As expected the bounds on output tensors are determined by the block count $K$ for MXNorm and by the hidden dimension $D$ for RMSNorm. The bounds are $O(\sqrt{K})$ for MXNorm(p=2) ($\sqrt{K}/c$ = $\sqrt{64}$/0.4688 = 17.06) and RMSNorm ($\sqrt{D}$ = 32), but $O(K)$ for MXNorm(p=1) ($K/c$ = 64/0.4814 = 132.95). However, we also note the much larger range on the \emph{input} to MXNorm(p=1). Larger values in norm output will contribute to larger weight updates, and we see this compounding throughout training. From this we conclude that the tighter bound on the output provided by RMSNorm and MXNorm(p=2) provides a beneficial stabilizing effect on training.

\begin{figure}
    \centering
    \includegraphics[width=\linewidth]{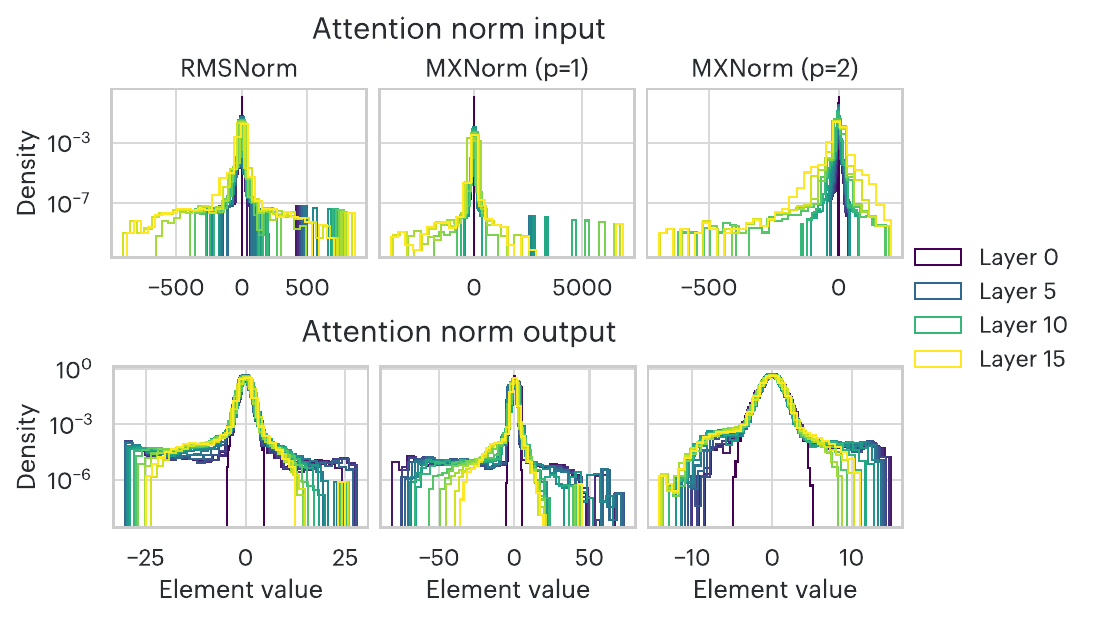}
    \caption{Histograms of input and output tensor elements for RMSNorm and MXNorm using mean or RMS over block absmaxes to compute normalization scale.}
    \label{fig:norm-io-hists}
\end{figure}

We note that the bounds of norm outputs being reached appears common, particularly in early layers. This indicates that worst-case input behaviour is common during training. Indeed, if one element in a token's hidden representation is much larger than all other elements (i.e. an outlier feature), then this closely approximates the worst case.

\subsection{Performance Analysis}

\begin{figure}[htbp]
    \centering
    \includegraphics[width=\linewidth]{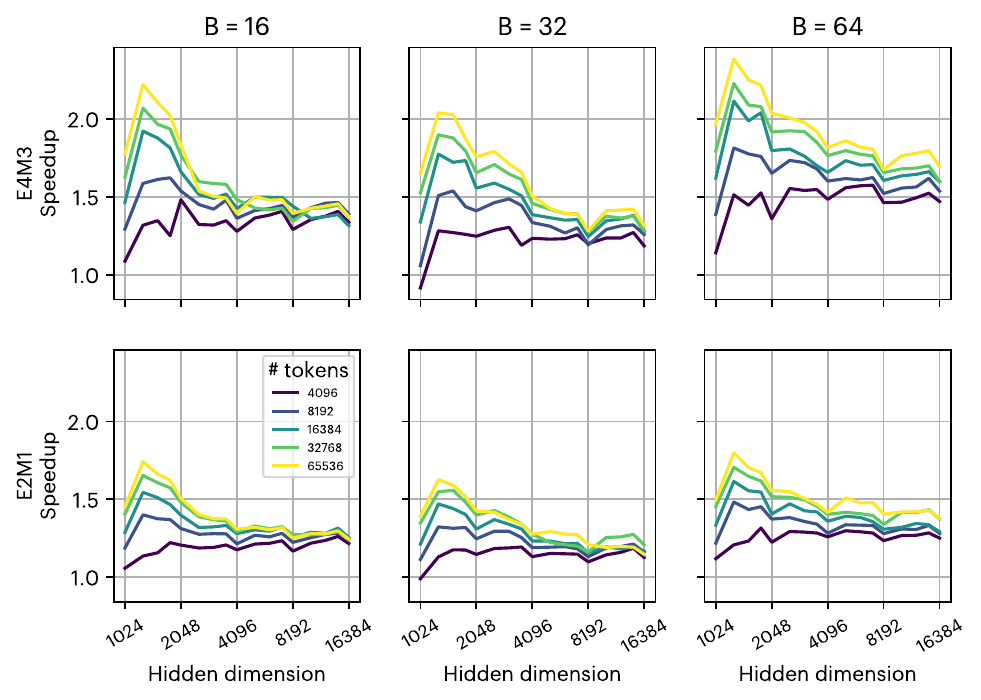}
    \caption{Relative speedup of isolated MXNorm over RMSNorm-MXCast. Normalization and casting have been fused together into a single kernel in both cases. $B$ denotes the block size.}
    \label{fig:perf-kernel}
\end{figure}

We examine the performance of an MXNorm(p=2) kernel compared with RMSNorm fused with an MXCast (Figure \ref{fig:perf-kernel}). Since MXNorm has the same gradient calculation as RMSNorm, we analyse only the forward pass. Additionally, we analyse in an inference-optimized setting in which the norm gain parameters have been fused with weights. In all cases, kernels were automatically generated with \texttt{torch.compile} on GB200 systems.

For our kernel benchmarks (Figure \ref{fig:perf-kernel}), we examine the speedup over different sizes of hidden dimension, token counts, MXFP block sizes and MXFP value datatypes. In almost all cases there is a speedup, with a maximum speedup of 2.4x. We find that the speedup generally increases with an increasing number of tokens, but converges towards a steady value with increasing hidden dimension. When taking the geometric mean to estimate the average case speedup across hidden dimension and token count (see \cite{fleming1986benchmarkstats} for why this is best practice) we found a 41.7\% speedup for the MXFP8 kernel ($B$=32, dtype=E4M3) and a 31.2\% speedup for the NVFP4 kernel ($B$=16, dtype=E2M1) on average. We examined these cases as they have dedicated support for accelerated matrix multiplication on current hardware. Furthermore, we find that these translate to geometric mean speedups of 1.3\% for transformer blocks of Llama 3 8B (hidden dimension 4096) in MXFP8 and 2.6\% in NVFP4, which is consistent with our claim that there is a greater need to optimize non-matmul operations as we move to lower-precision datatypes. Full details of these benchmarks can be found in Appendix \ref{apx:perf_experiment_details}.

\section{Conclusion} \label{sec:conclusion}

We demonstrate that RMSNorm can be replaced with MXNorm: a normalization scheme that reuses block absmaxes calculated during MXFP quantization. We show how to make minimal modifications with no extra hyperparameters to standard LLM architectures to use MXNorm in situations where normalization precedes linear layers. We find that pre-training stability correlates with the bound placed on extreme values by the normalization layer, and that MXNorm can provide similar bounding as RMSNorm as model size increases. We show that this approach preserves pre-training performance and zero-shot capability on downstream tasks for models of up to 8B parameters. We also show speedups on commercially available hardware with minimal software engineering effort.

We note that MXNorm can be used with even narrower value formats (e.g. INT2, ternary) where its advantage may be even greater and generalizes to other forms of quantization where the block absmaxes of activations are computed (e.g. VS-Quant \cite{dai2021vsquantpervectorscaledquantization}). We also note that normalization is just one of many neural network components that are not accelerated by improvements to matrix multiplication throughput in AI accelerators, such as the application of rotary positional encodings in attenton layers \cite{su2023roformerenhancedtransformerrotary} and gated linear units in feedforward layers \cite{shazeer2020gluvariantsimprovetransformer}. We leave the directions suggested by these observations to future work.

\clearpage
\printbibliography
\clearpage
\tableofcontents
\clearpage
\newpage


\appendix
\setcounter{figure}{0} 
\renewcommand{\thefigure}{A.\arabic{figure}} 

\section{Contribution and Acknowledgements} \label{sec:acknowledgements}

Paul Balan\c{c}a (PB) initially proposed the project. Callum McLean (CM) and Luke Y. Prince (LYP) developed the final MXNorm scheme. Alexandre Payot (AP) developed the majority of pre-training and evaluation infrastructure necessary for the project. CM and LYP ran experiments. CM and LYP wrote the paper. PB and Carlo Luschi supervised the project.

We thank Douglas Orr, Tom Cashman, and Andrew Fitzgibbon for useful discussions throughout the project and feedback on initial drafts of the paper. We thank Charlie Blake for insights leading to the start of the project.

\section{Compute resources} \label{apx:compute-resources}

Our 125M models are trained in 60 minutes on a single node with 8 NVIDIA H100s connected by NVLink. Our 1B models are trained in 6 hours on 4 nodes comprising 32 NVIDIA H100s connected by NVLink. Our 8B models are trained over 8 days on 8 nodes comprising 64 NVIDIA H100s connected by NVLink. We thank LambdaLabs for providing the compute for this project.

\section{MXFP Primer} \label{apx:mxfp-primer}

MX quantization chunks a tensor into contiguous blocks of a fixed size and computes a scale factor for each block, which is used to rescale the elements of the block to the range of a low-precision format. An MX tensor can therefore be thought of as a tuple of an MX scale tensor comprised of block scales and an MX values tensor of rescaled, quantized elements. Using the E8M0 format (which only represents integer powers of 2) for the scales preserves the range of BF16 while adding the minimal number of bits to the representation. The scale for each block is thus chosen to be the block's absmax rescaled to use the full range of the MX values tensor datatype (e.g., E4M3), then rounded to a power of 2 \cite{rouhani2023microscalingdataformatsdeep}. The implementation of this rounding is not fully standardized and several schemes have been proposed \cite{pytorch_ao_mx_config_2024, mishra2025recipespretrainingllmsmxfp8, or2025torchaopytorchnativetrainingtoservingmodel}. 
We illustrate the design for an example MXFP8 format in Figure \ref{fig:mx-block}. 

\begin{figure}[htbp]
    \centering
    \includegraphics[width=\linewidth]{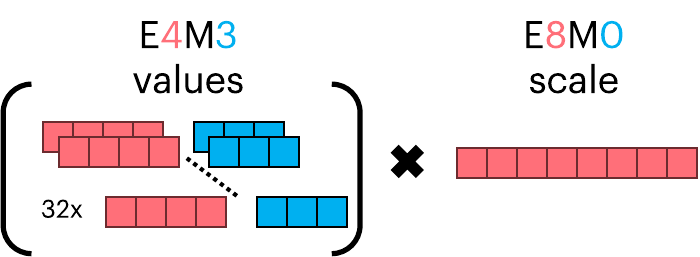}
    \caption{Representation of a block of 32 elements stored in MXFP8. $\text{E}x\text{M}y$ denotes the $x$ bits allocated to \textbf{E}xponent and $y$ bits allocated to \textbf{M}antissa. Each contiguous block of 32 elements in an MXFP8 tensor is represented by a tuple of 32 E4M3 values along with an unsigned E8M0 scale shared among the 32 elements. During dequantization, each value is multiplied by its corresponding shared scale.}
    \label{fig:mx-block}
\end{figure}

\section{Additional information for Theorem \ref{thm:mean-absmax-almost-sure-rms-ratio}} \label{apx:proof-mean-absmax-almost-sure-rms-ratio}

\subsection{The continuous-mapping theorem}
\begin{theorem}

If $(A_n) \in \mathbb{R}^\ell$ is a sequence of random variables such that $A_n \xrightarrow{\mathrm{a.s.}} A$ and $f:\mathbb{R}^\ell \rightarrow \mathbb{R}$ is continuous everywhere in some $C \subseteq \mathbb{R}^\ell$ such that $\mathbb{P}(A_n \in C) = 1$, then $f(A_n) \xrightarrow{\mathrm{a.s.}} f(A)$.

\end{theorem}

\begin{proof}
See e.g. page 7 of \cite{van2000asymptotic}.
\end{proof}

\subsection{Note on the assumption that $\mathbb{P}(Z = 0) = 0$}

The assumption that $\mathbb{P}(Z = 0) = 0$ is justified if $Z$ is drawn from a continuous distribution such as a Gaussian or uniform distribution. If this distribution is quantized to floating point, this assumption may no longer hold. However, if the RMS of a vector is 0 then all of the elements must be 0, and so the generalized power norm of the block absmaxes is also 0 and hence our estimate will be accurate. We do not expect this to matter often in practical neural network training.

\section{RMSNorm gradients} \label{apx:rmsnorm-bwd}

Given an activation tensor $X \in \mathbb{R}^{N \times D}$, RMSNorm normalises each token of the tensor $X$ using the calculated RMS $S$, given by:
\begin{align}
S_i &= \sqrt{\frac{1}{D}\sum_{k=1}^D X_{ik}^2} \label{eqn:apx-rms}
\end{align}
The learnable gain parameter $\gamma \in \mathbb{R}^D$ rescales the normalised $X$ along the hidden dimension to give output $Z$ as follows:
\begin{align}
    Z_{ij} &= (X_{ij}/S_i)\gamma_j \label{eqn:apx-gain}
\end{align}

As mentioned in Section \ref{sec:mxnormlinear}, we use the gradient calculation of RMSNorm as a straight-through estimator for MXNorm. For reference, here we provide a definition of our gradient calculation. Given the gradient of the loss with respect to the linear layer output $\nabla Y$, the backwards pass for RMSNorm followed by a linear layer is given by:

\begin{align}
    (\nabla Z) &= (\nabla Y) W \\
    (\nabla W) &= (\nabla Y)^{\top} Z \\
    (\nabla \gamma)_j &= \sum_k(\bar{X}_{kj} \cdot (\nabla Z)_{kj}) \\
    (\nabla \bar{X}) &= (\nabla Z)\gamma \\
    U_{ij} &= \bigg( \sum_k(\nabla \bar{X})_{ik}X_{ik} \bigg) X_{ij}\\
    (\nabla X) &= (S^{-1})^\top \nabla \bar{X} - \frac{1}{D}(S^{-3})^\top U 
\end{align}

where we define $\bar{X}$ such that $\bar{X}_{ij} = X_{ij}/S_i$, and define $S^k$ by $(S^k)_i = (S_i)^k$.

\section{Implementation of MXNormLinear} \label{apx:impl}

We provide a PyTorch implementation for the forward and backward pass of MXNormLinear using the Pre-Norm scheme. We omit the details of \texttt{pow2\_round} since there is no standardized implementation \cite{or2025torchaopytorchnativetrainingtoservingmodel}. In our experiments we use the implementation defined by \cite{mishra2025recipespretrainingllmsmxfp8}.

\begin{lstlisting}[style=python,basicstyle=\footnotesize\ttfamily]
import torch
from torchao.prototype.mx_formats.mx_tensor import MXTensor
from math import log2, floor
from mx_norm_utils import pow2_round

def absmax_scale_factor(block_size, p):
    """
    Estimated RMS(X)/G(max(abs(X))) using Monte Carlo
    sampling from Gaussian distribution
    """
    ABSMAX_SCALE_FACTOR = {
        (16, 1): 0.4814,
        (16, 2): 0.4688,
        (32, 1): 0.4261,
        (32, 2): 0.4185,
        (64, 1): 0.3852,
        (64, 2): 0.3803
    }
    return ABSMAX_SCALE_FACTOR[(block_size, p)]

def get_largest_pow2(dtype):
    return 2 ** floor(log2(torch.finfo(dtype).max))

def mx_quantise(x, mx_data_dtype, block_size):
    x_blocked = x.reshape(*x.shape[:-1], -1, block_size)
    block_absmax = x_blocked.abs().amax(dim=-1, keepdim=True)
    largest_pow2 = get_largest_pow2(mx_data_dtype)
    mx_scales = pow2_round(block_absmax / largest_pow2)
    mx_data = x_blocked / mx_scales.unsqueeze(-1)
    mx_data = mx_data.reshape(x.shape)
    mx_data = mx_data.to(mx_data_dtype)
    mx_scales = mx_scales.to(torch.float8_e8m0fnu)
    return MXTensor(
            mx_scales, 
            mx_data, 
            mx_data_dtype, 
            block_size, 
            x.dtype
    )

def p_mean(x, p, dim=-1, keepdim=True):
    return x.pow(p).mean(dim=dim, keepdim=keepdim).pow(1/p)

def mx_norm(x, mx_data_dtype, block_size, p):
    x_blocked = x.reshape(*x.shape[:-1], -1, block_size)
    block_absmax = x_blocked.abs().amax(dim=-1, keepdim=True)

    # absmax_scale_factor returns the expected ratio 
    # of the RMS divided by the mean of the absmaxes
    coef = absmax_scale_factor(block_size, p)
    
    rms_estimate = p_mean(block_absmax, p=p) * coef
    scaled_block_absmax = block_absmax / rms_estimate
    largest_pow2 = get_largest_pow2(mx_data_dtype)
    mx_scales = pow2_round(scaled_block_absmax / largest_pow2)

    # Creating data tensor: 
    # Want mx_scale * mx_data = x / rms_estimate
    # So we want mx_data = (x / rms_estimate) / mx_scale
    # Need to cast MX scales back to match dtypes for divide
    mx_data = (x_blocked / rms_estimate) / mx_scales
    mx_data = mx_data.reshape(x.shape)
    mx_data = mx_data.to(mx_data_dtype)
    return MXTensor(
            mx_scales,
            mx_data, 
            mx_data_dtype, 
            block_size, 
            x.dtype
    )

def mx_norm_linear_forward(
        x, 
        norm_weight, 
        linear_weight,
        mx_data_dtype,
        block_size,
        p
    ):    
    mx_normalised_activations = mx_norm(
        x, mx_data_dtype, block_size, p
    )
    mx_fused_weight = mx_quantise(
        norm_weight * linear_weight, mx_data_dtype, block_size
    )
    return torch.mm(
        mx_normalised_activations, mx_fused_weight.t()
    )

def rms_norm_grad(grad_out, x, rms):
    delta = torch.mean(grad_out * x, dim=-1, keepdim=True)
    grad_x = rms.pow(-1) * grad_out - rms.pow(-3) * x * delta
    return grad_x

def mx_norm_linear_backward(
        grad_out, 
        rms_estimate, 
        x, 
        norm_weight, 
        linear_weight, 
        mx_data_dtype, 
        block_size
    ):
    # Need grad_out to be mx_quantised along both 
    # rows and columns
    mx_grad_out = mx_quantise(
        grad_out, mx_data_dtype, block_size
    )
    mx_t_grad_out = mx_quantise(
        grad_out.t(), mx_data_dtype, block_size
    )
    
    # Divide by rms_estimate from forward pass
    # in higher precision
    normed_x = x / rms_estimate
    mx_t_normed_x = mx_quantise(
        normed_x.t(), mx_data_dtype, block_size
    )
    mx_linear_weight = mx_quantise(
        linear_weight, mx_data_dtype, block_size
    )

    # Compute parameter gradients
    grad_mm_input = torch.mm(
        mx_grad_out,  mx_linear_weight.t()
    )
    grad_linear_weight = torch.mm(
        mx_t_grad_out, mx_t_normed_x.t()) * norm_weight
    grad_norm_weight = torch.sum(
        normed_x * grad_mm_input, dim=0
    )
    
    # Compute gradient w.r.t input
    # Use rms_norm_grad as a straight through 
    # estimator for mx_norm 
    grad_normed_x = grad_mm_input * norm_weight
    grad_x = rms_norm_grad(
        grad_normed_x, x, rms_estimate
    )
    return grad_x, grad_norm_weight, grad_linear_weight
\end{lstlisting}

\section{Training Experiment Details} \label{apx:experiment_details}

For all experiments we use the TorchTitan distributed pre-training library \cite{liang2025torchtitanonestoppytorchnative} with FSDP in conjunction with TorchAO \cite{or2025torchaopytorchnativetrainingtoservingmodel} for MX quantization. We used model sizes defined by configurations in Table \ref{apx:experiment_details} for all of our experiments, except where specified. We keep all other hyperparameters from TorchTitan the same.

\begin{table}[h]
\centering
 \caption{Pretraining configuration for Llama 3 models}
\begin{tabular}{lccc}
  &\multicolumn{3}{c}{Model Size}\\
      &125M&  1B& 8B\\
      \hline \\
      Global batch size&32& 
256& 256\\
 Sequence length& 4096 & 4096 & 4096\\
 Total training tokens& 2.62B & 20.97B & 314.57B\\
 Transformer block count & 8 &16 & 32\\
 Hidden dimension& 1024 & 2048 & 4096\\
 Q:KV head ratio& 4:1 & 4:1 & 4:1\\
 FFN dimension& 1536 & 3072 & 7168\\ 
 \end{tabular}
    \label{tab:configs}
\end{table}

\section{Performance Benchmark Details} \label{apx:perf_experiment_details}

\subsection{Kernel Benchmarks}

We implement two PyTorch modules, one that returns the MXFP output of MXNorm and one that normalizes the input tensor with RMSNorm and casts the normalized tensor to MXFP. We compile each module with \texttt{torch.compile} and use the facilities provided by \texttt{torch.profiler} to profile the execution on a random Gaussian input of the appropriate shape on an NVIDIA GB200. We compute the sum of the durations of the CUDA events to remove noise from the overhead of launching CUDA kernels. For each norm implementation, we profile each combination of the possibilities listed in Table \ref{tab:kernel-perf-configs} and plot the speedup.

\begin{table}[h]
\centering
 \caption{Kernel benchmark cases. All possible combinations are profiled.}
\begin{tabular}{lc}
      &Set of tested values\\
      \hline \\
      Tokens&$\{ 2^{12}, 2^{13}, 2^{14}, 2^{15}, 2^{16} \}$\\
 Hidden dim& $\{k \cdot 2^r : k \in \{1, 3, 5, 7\}, r \in \mathbb{N} \} \cap [2^{10}, 2^{14}]$\\
 Block size& $\{16, 32, 64\}$\\
 MXFP dtype & \{E4M3, E2M1\} \\
 \end{tabular}
    \label{tab:kernel-perf-configs}
\end{table}

The set of hidden dimensions profiled can also be described as the powers of 2 from $2^{10}$ to $2^{14}$, as well as three evenly spaced numbers between each pair of adjacent powers of 2.

\subsection{Full Layer Benchmarks}

For both RMSNorm and MXNorm, we profile each possible combination listed in Table \ref{tab:layer-perf-configs} (with the sequence length computed as the number of tokens divided by the batch size). 
For each combination we profile the forward pass of a Llama 3 8B model (with 8 transformer layers rather than 32) on an NVIDIA GB200 inside \texttt{torch.no\_grad()}, using \texttt{torch.compile} to compile the model and the facilities provided by \texttt{torch.profiler} to export a Chrome trace as JSON.
Linear layers in the model are performed using kernels targeting the appropriate dedicated Tensor Cores on GB200 (accessed from Python via TorchAO \cite{or2025torchaopytorchnativetrainingtoservingmodel}, with the original kernels coming from cuBLAS for MXFP8 or CUTLASS \cite{Thakkar_CUTLASS_2023} for NVFP4).

To remove noise from the overhead of launching CUDA kernels, we identify the first and last CUDA event of each transformer layer in the Chrome trace and compute the sum of the durations of the events in each transformer layer. We generate 10 profiles for each combination and compute the mean time to execute a transformer layer across all runs (discarding the first and second layers, where the embedding layer can be fused into the residual add and norm).
We compute the speedup of MXNorm over RMSNorm for each combination and report the geometric mean of speedups for each data type.

\begin{table}[h]
\centering
 \caption{Layer benchmark cases}
\begin{tabular}{lccc}
      &Set of tested values\\
      \hline \\
      Batch size&$\{4, 8\}$\\
      Tokens&$\{ 2^{12}, 2^{13}, 2^{14}, 2^{15}, 2^{16} \}$\\
 MXFP type & $\{\text{MXFP8}, \text{NVFP4}\}$ \\
 \end{tabular}
    \label{tab:layer-perf-configs}
\end{table}

\section{Convergence of Llama 3 1B with MXNorm} \label{apx:convergence}

In Section \ref{subsec:pretraining} we demonstrate the learning rate sensitivity of small Llama 3 models with RMSNorm and MXNorm. We illustrate the training loss curves to show similarity in convergence behaviour of models with the lowest training loss for each category. 

\begin{figure}[h]
    \centering
    \includegraphics[width=\linewidth]{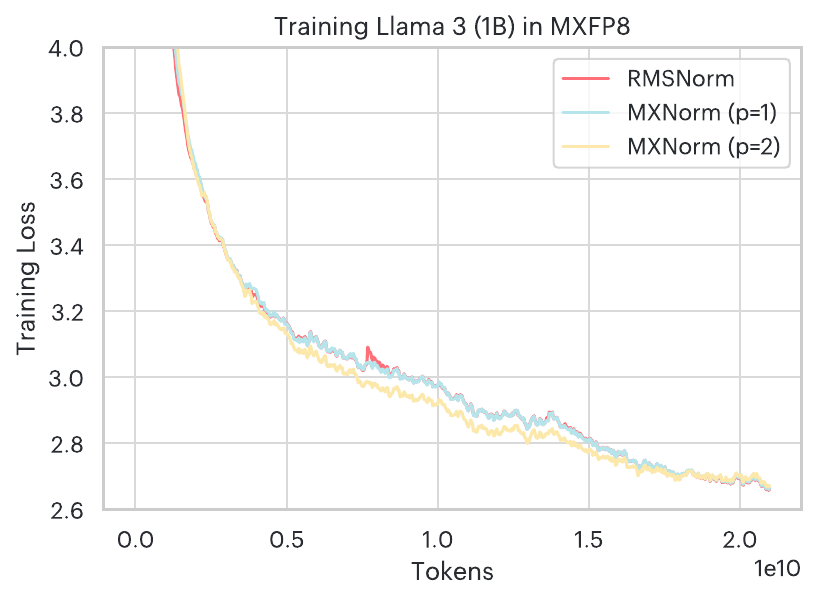}
    \caption{Training loss convergence of 1B parameter models trained on 20B tokens with MXNorm and RMSNorm.}
    \label{fig:training-loss-1b}
\end{figure}

\section{Spike score} \label{apx:spike-score}

To allow us to investigate sources and possible fixes of loss spikes we define a spike score and use this to guide the search for a proxy model that could help us understand why MXNorm(p=1) suffers from loss spikes at 8B parameter scale, but MXNorm(p=2) and RMSNorm do not. 

We define a spike score as "the area under a detrended, normalized training loss curve above 3 standard deviations", i.e., for loss $\mathcal{L}$ at step $t$

\begin{equation}
    \text{Spike Score} = \sum_t \max \left( \frac{\mathcal{L}_t - m_t}{\sigma}  - 3, 0 \right)
\end{equation}

where $m_t$ is the rolling minimum of $\mathcal{L}_t$ over 8 steps and $\sigma$ is the standard deviation in the loss observed over the last 30\% of steps.

We show the distribution of spike scores across 8 training runs for our depth=16, width=1024, MXFP block size 16 proxy model for each condition in Figure \ref{fig:spike-scores}.

\begin{figure}[htbp]
    \centering
    \includegraphics[width=
    \linewidth]{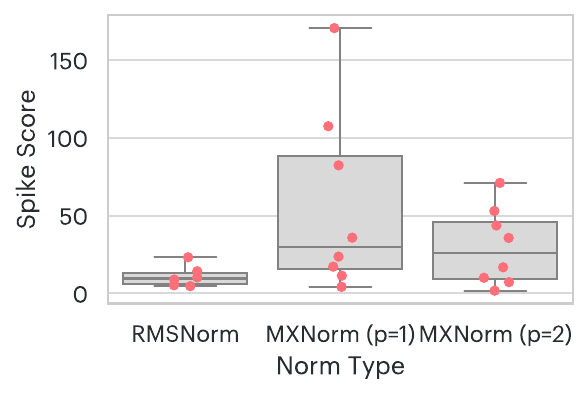}
    \caption{Boxplot of spike scores showing the mean, interquartile range, and range of values over 8 training runs. Spike score is defined as the area under a detrended loss curve above 3 standard deviations.}
    \label{fig:spike-scores}
\end{figure}

\section{Post-round MXNorm} \label{apx:post-round}

We define an ablation to our design in which we estimate the RMS of the input using the block scales after they have been rounded to E8M0 (Post-round MXNorm), and demonstrate its instability at 8B parameter scale.

\subsection{Approximating the average of the rounded scale factors}
We wish to calculate the expected value of the rounded MX scale factors assuming the tensor is drawn from a Gaussian distribution with mean zero and unknown standard deviation. Suppose
$ X_i \sim \mathcal{N}(0, \sigma^2) $
for $ 0 \le i < K $ where $K$ is the MX block size and the $X_i$'s are independent.
We define
$ r(x) := 2^{\left\lfloor log_2(x) \right\rfloor} $
for the rounded MX scale factors in E8M0 (ignoring the \texttt{rescale} operation in Equation \ref{eqn:blockscale}, which amounts to multiplication by a constant and hence can be factored out of the calculations).

We then have
$$ \mathbb{E}(\max_i(r(X_i)))
= \sum_{j=-\infty}^{\infty}2^j \cdot \mathbb{P}(\max_i(r(X_i)) = 2^j)
$$

We therefore need to compute $ \mathbb{P}(\max_i(r(X_i)) = 2^j) $. We have
$$ \mathbb{P}(\max_i(r(X_i)) = 2^j) $$
$$ = \mathbb{P}(\exists i : r(X_i) = 2^j \land \forall i : r(X_i) \le 2^j) $$
$$ = \mathbb{P}(\exists i : 2^j \le |X_i| < 2^{j+1} \land \forall i : |X_i| < 2^{j+1}) $$

Applying the law of probability that $ \mathbb{P}(A \land B) = \mathbb{P}(B) \cdot \mathbb{P}(A|B) $ gives:

$$ = \mathbb{P}(\forall i : |X_i| < 2^{j+1}) \cdot \mathbb{P}\bigg(\exists i : 2^j \le |X_i| < 2^{j+1} \bigg\rvert \forall i : |X_i| < 2^{j+1}\bigg) $$

$$ = \mathbb{P}(|X_0| < 2^{j+1})^K \cdot \bigg(1 - \mathbb{P}\bigg(\forall i : |X_i| < 2^j \biggr\rvert \forall i : |X_i| < 2^{j+1}\bigg)\bigg) $$

The simplification of the left-hand term of the product comes from the fact that the $X_i$'s are independent and identically distributed (the choice of $X_0$ is arbitrary).

Applying the law of probability that if $ C \Rightarrow D $ we have $ \mathbb{P}(C|D) = \mathbb{P}(C \cup D) / \mathbb{P}(D) = \mathbb{P}(C) / \mathbb{P}(D) $ gives:

$$ = \mathbb{P}(|X_0| < 2^{j+1})^K \cdot \bigg(1 - \frac{\mathbb{P}(|X_0| < 2^{j})^K}{\mathbb{P}(|X_0| < 2^{j+1})^N}\bigg) $$
$$ = \mathbb{P}(|X_0| < 2^j)^K - \mathbb{P}(|X_0| < 2^{j+1})^K $$

Note that if $ X_0 \sim \mathcal{N}(0, \sigma^2) $, then (using the symmetry of the Gaussian distribution):

$$ \mathbb{P}(|X_0| < x) = \mathbb{P}(X_0 < x) - \mathbb{P}(X_0 \le -x) = \mathbb{P}(X_0 < x) - (1 - \mathbb{P}(X_0 \le x)) $$
$$ = \mathbb{P}(X_0 < x) - (1 - \mathbb{P}(X_0 < x)) = 2 \cdot \mathbb{P}(X_0 < x) - 1 $$

Using the CDF of the standard normal distribution $\Phi(\cdot)$, we therefore have the following:

$$ \mathbb{E}(\max_i(r(X_i))) = \sum^{\infty}_{j=-\infty} 2^j \cdot \bigg( \bigg(2 \cdot \Phi( 2^j/\sigma) - 1\bigg)^K - \bigg(2 \cdot \Phi( 2^{j+1}/\sigma) - 1\bigg)^K \bigg) $$

As $|j|$ increases, the term in the sum rapidly decreases, so the sum can be truncated from a sum over $j \in \mathbb{N}$ to a sum over $-J < j < J$ for large $J$ with little loss in accuracy. This truncated sum can then be calculated using any mathematical software package that supports evaluating the CDF of the standard normal distribution (including PyTorch).

\subsection{Approximating the RMS from the MX tensor}

For a fixed block size $K$, we define $f: \mathbb{R}^+ \rightarrow \mathbb{R}^+$ to be given by $f(\sigma) = \mathbb{E}(\max_i(r(X_i)))$ when $X_i \sim \mathcal{N}(0, \sigma^2)$ for $0 \le i < K$ and the $X_i$'s are independent.

We observe that $f$ is strictly increasing since if the standard deviation is greater we expect the rounded scale factors to be greater, and therefore $f$ is invertible. Hence given the mean of the rounded scale factors $\bar{X}$ we can approximate the RMS of the original values as $f^{-1}(\bar{X})$. 

Since $f$ is strictly increasing, we can compute $f^{-1}(\bar{X})$ to an arbitrary degree of precision by finding $\sigma_1, \sigma_2$ such that $f(\sigma_1) < \bar{X} < f(\sigma_2)$ and then iteratively narrowing this range using a binary search.

However, this is computationally expensive to do at every layer in a model. We observe that $f(2 \sigma) = 2 f(\sigma)$ (since $r(2x) = 2r(x)$ and $\max_i|2X_i| = 2\max_i|X_i|$). Thus we can pre-compute $f^{-1}(2^{i/A})$ for $0 \le i \le A$ for some $A$ and approximate $f^{-1}$ on the interval $[1, 2]$ using linear interpolation on these pre-computed values, and approximate $f^{-1}(x)$ elsewhere using the identity $f^{-1}(x) = 2^{-k}f^{-1}(2^kx)$, where $k$ is chosen such that $2^kx \in [1, 2]$.

\subsection{Post-Round MXNorm Results}

\begin{figure}[htbp]
    \centering
    \includegraphics[width=\linewidth]{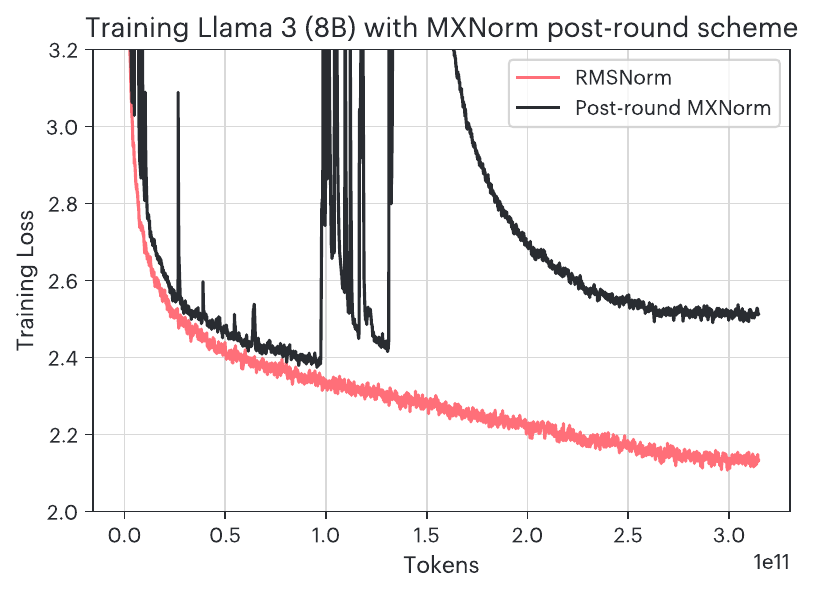}
    \caption{The post-round MXNorm scheme training loss shows instability that significantly diverges from the RMSNorm baseline.}
    \label{fig:post-round}
\end{figure}

\end{document}